\documentclass[conference]{IEEEtran}
\IEEEoverridecommandlockouts
\usepackage{cite}
\usepackage{amsmath,amssymb,amsfonts}
\usepackage{graphicx}
\usepackage{textcomp}
\usepackage{xcolor}
\usepackage{booktabs}
\usepackage{multirow}
\usepackage{array}
\usepackage{url}
\usepackage{amsthm}
\usepackage{algorithm}
\usepackage{algpseudocode}
\usepackage{tikz}
\usepackage{pgfplots}
\pgfplotsset{compat=1.18}
\usetikzlibrary{arrows.meta,positioning,shapes.geometric,
                fit,backgrounds,calc,decorations.pathreplacing}

\newtheorem{proposition}{Proposition}
\newtheorem{corollary}{Corollary}

\def\BibTeX{{\rm B\kern-.05em{\sc i\kern-.025em b}\kern-.08em
    T\kern-.1667em\lower.7ex\hbox{E}\kern-.125emX}}

\newcommand{\eps}{\boldsymbol{\varepsilon}}
\newcommand{\Dc}{\boldsymbol{\Delta}}
\newcommand{\tileps}{\tilde{\boldsymbol{\varepsilon}}}
\newcommand{\norm}[1]{\left\|#1\right\|}
\newcommand{\dv}{\boldsymbol{\delta}}

\definecolor{s1col}{RGB}{14,116,172}
\definecolor{s2col}{RGB}{107,68,168}
\definecolor{s3col}{RGB}{210,100,20}
\definecolor{s4col}{RGB}{34,130,72}
\definecolor{s5col}{RGB}{16,130,130}

\begin{document}

\title{Closing the Null Space: Guidance-Aware\\
Quantization for Classifier-Free Diffusion}

\author{%
\IEEEauthorblockN{Abdullah Al Shafi,~Sumaiya Rahim Suma}
\IEEEauthorblockA{\textit{Department of Computer Science and Engineering}\\
\textit{Khulna University of Engineering \& Technology}\\
Khulna-9203, Bangladesh\\
\{abdullah.shafi99, sumaiya.rahim234\}@gmail.com}
}

\maketitle

\begin{abstract}
Deploying classifier-free guidance (CFG) diffusion models under
real-world compute budgets requires quantization, yet existing
post-training quantization (PTQ) methods treat CFG models as
single-branch networks, ignoring the paired conditional/unconditional
structure that CFG inference fundamentally relies on.
This structural blind spot has two consequences.
At the system level, the two-pass CFG execution pattern imposes a
latency overhead that parameter-count and bit-operation metrics
conceal entirely, and commodity INT8 inference stacks fail to
realize the theoretical efficiency gains that BOPs calculations
promise.
At the algorithmic level, calibrating against the guidance gap
alone admits an exact null space: a quantized model can achieve
perfect gap-fidelity diagnostics while the unconditional branch
drifts arbitrarily, corrupting every guided prediction at inference
time.
This paper terms this the \emph{branch-drift trap}, proves its
existence analytically, and confirms it empirically through a
false-positive result in which the best-calibrated model
by standard diagnostics simultaneously produces the worst sample
quality.
To close the trap, Guidance-Aware Mixed Precision (GAMP) is
proposed, which calibrates directly on the guided prediction,
derives per-layer activation-bit sensitivity from guided-output
degradation, and allocates bits via a greedy knapsack---provably
preventing unconditional branch drift by construction.
\end{abstract}


\section{Introduction}
\label{sec:intro}

Diffusion models~\cite{ho2020ddpm} have become the dominant
framework for high-quality conditional image generation, with
classifier-free guidance (CFG)~\cite{ho2022cfg} providing the
primary control mechanism at inference time.
Deploying these models under memory and latency budgets motivates
post-training quantization (PTQ), which compresses a trained model
by quantizing its weights and activations without retraining.
CFG inference, however, differs fundamentally from the
unconditional generation setting that virtually all PTQ methods
assume: at every denoising step, the model executes \emph{two}
forward passes---one conditioned on a class label and one
on a null token---whose outputs are combined as
\begin{equation}
  \tileps = \eps_\emptyset + w(\eps_y - \eps_\emptyset), \quad w > 1.
  \label{eq:cfg}
\end{equation}
This dual-branch structure creates two deployment challenges that
the existing PTQ literature has not addressed.

The first is a measurement problem.
Efficiency metrics such as parameter counts and bit-operations (BOPs)
report single-pass costs, invisibly discounting the two-pass
per-step overhead that CFG imposes.
The same measurement gap affects INT8 deployment: BOPs-predicted
speedups assume hardware-matched inference stacks, but commodity
ONNX Runtime without TensorRT routes quantized operators through
fallback pathways that negate---and in practice reverse---the
theoretical gains.

The second is a calibration problem, and it is structural.
A natural PTQ objective for a CFG model is to calibrate the
guidance gap $\Dc = \eps_y - \eps_\emptyset$, since DASH's
distillation work~\cite{dash2026} demonstrates that preserving
$\Dc$ is central to generation quality.
This paper shows, however, that a gap-only calibration objective
admits an exact null space: there exist quantized models that
achieve perfect gap fidelity---near-ideal $\rho$ and
$\cos(\Dc)$ scores---while the unconditional branch drifts by an
unconstrained vector $\dv$, corrupting the guided prediction at
inference time.
This phenomenon is termed the \emph{branch-drift trap}.
Unlike the training-time setting in DASH, where an explicit
unconditional branch loss and hundreds of thousands of gradient
steps over unfrozen weights close this null space automatically,
PTQ calibrates only quantization thresholds over frozen weights in
a few hundred iterations---a regime in which no gradient mechanism
exists to suppress $\dv$.

To address both challenges, this paper makes four contributions.
\begin{enumerate}
\item \textit{CFG latency characterization}: the two-pass
  guided-inference overhead is measured empirically on an NVIDIA
  T4 GPU across batch sizes and floating-point formats, revealing
  a consistent 1.99$\times$ tax hidden by all standard efficiency
  proxies and diagnosing a precise software-stack root cause for
  the failure of commodity INT8 inference to realize theoretical
  speedups.
\item \textit{Null-space theorem and empirical confirmation}: the
  branch-drift trap is proved formally---showing that gap-only PTQ
  calibration admits an exact null space in which guided-prediction
  error is unconstrained---and confirmed empirically through a
  controlled ablation in which the best-performing model by gap
  diagnostics simultaneously produces the worst sample quality of
  any tested configuration.
\item \textit{Guidance-Aware Mixed Precision (GAMP)}: a PTQ method
  that calibrates on the guided prediction directly, closes the
  null space by construction, and allocates activation bits per
  layer according to their contribution to guided-output
  degradation---a CFG-aware sensitivity criterion that the existing
  single-branch mixed-precision literature does not capture.
\item \textit{Deployment guidelines}: concrete recommendations for
  HPEC practitioners on how to report efficiency, validate PTQ
  calibration, and allocate activation precision for CFG diffusion
  deployment.
\end{enumerate}

\section{Related Work}
\label{sec:related}

\textit{Post-training quantization for diffusion models.}
PTQ4DM~\cite{ptq4dm} established the feasibility of diffusion PTQ
by collecting multi-timestep calibration data to address time-varying
activation statistics.
Q-Diffusion~\cite{qdiffusion} introduced skip-connection partitioning
and block-wise reconstruction; for text-conditioned models it collects
features from both CFG branches, but its reconstruction objective is
per-block MSE on each branch independently, not the guided prediction
$\tileps = \eps_\emptyset + w\Dc$.
PTQD~\cite{ptqd} modeled quantization noise as a perturbation of the
diffusion SDE, proposing timestep-aware mixed-precision that targets
trajectory accumulation---a complementary error source to the
branch-level failure identified here.
TFMQ-DM~\cite{tfmqdm}, APQ-DM~\cite{apqdm}, and
TCAQ-DM~\cite{tcaqdm} introduced further refinements to temporal
feature maintenance and per-channel calibration.
Despite these advances, no existing method calibrates against
$\tileps$ directly: all treat the CFG model as a single-pass
unconditional network during calibration, leaving the branch-drift
trap structurally unaddressed.
The failure mode is architecture-agnostic---it depends only on the
linear combination in~\eqref{eq:cfg} and applies equally to
convolutional UNets and transformer-based DiT architectures.

Existing per-block MSE or per-tensor reconstruction losses, even when
fed with dual-branch features, remain gap-like objectives in the
terminal output space: they minimize branch-wise discrepancies
independently.
Because no explicit coupling term constrains $\dv$, the null space
of Proposition~\ref{prop:nullspace} remains reachable---a constant
branch offset propagates through the network and achieves low
per-block error on both branches simultaneously while severely
corrupting $\tileps$.

\textit{Mixed-precision allocation.}
BRECQ~\cite{brecq} and HAWQ~\cite{hawq} allocate bit-widths by
per-layer sensitivity derived from block-wise reconstruction error or
Hessian curvature---single-pass proxies that do not account for
CFG's dual-branch structure.
GAMP instead measures sensitivity through
$\norm{\tileps^q_l - \tileps^{fp}_l}^2$, the quantity CFG sampling
depends on directly, targeting the deployment-critical failure mode
rather than a reconstruction proxy.

\textit{CFG distillation and the DASH null space.}
DASH~\cite{dash2026} compresses a 35.8M-parameter CFG teacher to a
6.1M student via distillation, introducing $\rho$ and $\cos(\Dc)$
as guidance-fidelity diagnostics and proving a training-time null
space for gap matching under full gradient access.
The present work extends this analysis to the PTQ regime, where
frozen weights make the null space provably impossible to escape via
gap calibration alone, and demonstrates that the diagnostics DASH
correctly uses for distillation monitoring become
\emph{insufficient}---and actively misleading---as PTQ calibration
objectives.

\section{Method}
\label{sec:method}

\subsection{Preliminaries}
\label{sec:prelim}

At each of $K\!=\!50$ DDIM~\cite{ddim} steps, two forward passes
through the UNet yield $\eps_y(x_t,t)$ (conditional, class label
$y$) and $\eps_\emptyset(x_t,t)$ (unconditional, null token),
combined by~\eqref{eq:cfg}.
The guidance gap $\Dc = \eps_y - \eps_\emptyset$ carries the class
signal; DASH~\cite{dash2026} quantifies its preservation by
$\rho \!=\! \mathbb{E}[\|\Dc_q\|]/\mathbb{E}[\|\Dc_{fp}\|]$
(magnitude ratio; ideal~1) and $\cos(\Dc_q, \Dc_{fp})$
(directional alignment; ideal~1).

Quantization thresholds are optimized via AdaRound-style~\cite{nagel2020adaround}
learnable weight rounding with MSE-optimal activation clipping,
over 800~iterations using 1\,024 calibration images drawn uniformly
from the CIFAR-10 training set; network weights are frozen
throughout.
Weights are quantized per-output-channel (symmetric, signed),
providing scale granularity that substantially reduces clipping error
at W4; activations are quantized per-tensor (asymmetric, unsigned)
with MSE-optimal clipping thresholds.
Biases and normalization parameters remain at FP32.
Bit-operations $\text{BOPs}\!=\!\sum_l M_l w_b a_b$ ($M_l$: layer
MACs, $w_b$: weight bits, $a_b$: activation bits) serve as the
hardware-agnostic compute metric.
The first and last UNet convolutions (\texttt{conv\_in},
\texttt{conv\_out}) remain FP32, consistent with standard diffusion
PTQ practice~\cite{qdiffusion}.

\subsection{The Branch-Drift Trap: Null-Space Analysis}
\label{sec:theory}

The most natural PTQ calibration objective for a CFG model is to
minimize the guidance-gap reconstruction error:
\begin{equation}
  \mathcal{L}_{\mathrm{gap}} =
  \norm{(\eps_y^q - \eps_\emptyset^q)
        - (\eps_y^{fp} - \eps_\emptyset^{fp})}^2.
  \label{eq:lgap}
\end{equation}
This is well-motivated: $\mathcal{L}_\text{gap}$ directly penalizes
drift in $\Dc$, the quantity DASH identifies as the primary quality
determinant.
Yet it is provably insufficient.

\begin{proposition}[Null space of gap-only PTQ]
\label{prop:nullspace}
For any vector $\dv$, the assignment
$\eps_y^q = \eps_y^{fp} + \dv$ and
$\eps_\emptyset^q = \eps_\emptyset^{fp} + \dv$
achieves $\mathcal{L}_{\mathrm{gap}} = 0$, leaving the unconditional
branch drift $\dv$ entirely unconstrained.
\end{proposition}
\begin{proof}
$(\eps_y^q - \eps_\emptyset^q)
 = (\eps_y^{fp}+\dv) - (\eps_\emptyset^{fp}+\dv)
 = \eps_y^{fp} - \eps_\emptyset^{fp}$,
so $\mathcal{L}_{\mathrm{gap}} = 0$ for any $\dv$.
\end{proof}

The consequence for sample quality follows immediately.

\begin{proposition}[Guided-prediction error under null-space drift]
\label{prop:error}
Under any null-space solution with drift $\dv$:
\begin{equation}
  \tileps^q - \tileps^{fp} = \dv,
  \label{eq:drift}
\end{equation}
independent of guidance scale $w$ and regardless of how well the gap
is calibrated.
\end{proposition}
\begin{proof}
Substitute $\Dc^q \!=\! \Dc^{fp}$ (Proposition~\ref{prop:nullspace})
and $\eps_\emptyset^q \!=\! \eps_\emptyset^{fp} + \dv$
into~\eqref{eq:cfg}:
$\tileps^q - \tileps^{fp}
 = [(\eps_\emptyset^{fp}+\dv) + w\Dc^{fp}]
   - [\eps_\emptyset^{fp} + w\Dc^{fp}]
 = \dv$.
\end{proof}

\begin{corollary}[Gap-fidelity metrics are insufficient]
\label{cor:metrics}
$\rho\!=\!1$ and $\cos(\Dc_q,\Dc_{fp})\!=\!1$ are simultaneously
achievable with $\norm{\dv} \gg 0$.
At each of the $K$ DDIM steps, Proposition~\ref{prop:error}
contributes $\norm{\dv}$ of per-step error; over the full trajectory,
$\ell_2$ error accumulates as $\mathcal{O}(K\norm{\dv})$.
Guidance-gap metrics are \emph{necessary but not sufficient} for
PTQ sample quality under CFG.
\end{corollary}

PTQ provides no mechanism to suppress this drift.
In DASH's distillation setting~\cite{dash2026}, an explicit
unconditional branch loss
$\mathcal{L}_{un}\!=\!\norm{\eps_\emptyset^{\text{student}} - \eps_\emptyset^{\text{teacher}}}^2$
applied over $>$200K gradient iterations on unfrozen weights directly
suppresses $\dv$.
In PTQ, threshold optimization has only $\mathcal{O}(L)$ degrees of
freedom (one clipping range per quantizable layer), and the gap loss
reaches $\mathcal{L}_{\mathrm{gap}}\!\approx\!0$ trivially by
shifting both branch outputs by a common offset without genuinely
aligning either branch to its full-precision reference.
The branch-drift trap is therefore a structural property of any
gap-only PTQ objective, not an optimization pathology.

At larger guidance scales ($w\!\sim\!7$--15 in text-to-image
generation), the guided-prediction magnitude grows as
$|\tileps| \propto w\|\Dc\|$ while $\dv$ remains fixed in absolute
magnitude by~\eqref{eq:drift}, making the drift a proportionally
larger directional contamination of the sampling score.
The branch-drift trap is therefore expected to intensify at
production guidance scales.

Propositions~\ref{prop:nullspace}--\ref{prop:error} jointly imply
that the PTQ objective must constrain the guided prediction
$\tileps = \eps_\emptyset + w\Dc$ directly.
The per-layer guided-output error is defined as:
\begin{equation}
  \mathcal{E}(l) = \norm{\tileps^q_l - \tileps^{fp}_l}^2,
  \label{eq:guide_err}
\end{equation}
where $l$ indexes a quantizable \texttt{Conv2d} or \texttt{Linear}
layer; $\tileps^q_l$ is the full-network guided output when layer
$l$'s quantized thresholds are active---no intermediate feature-map
hooks are required.
$\mathcal{E}(l)\!=\!0$ requires $\dv\!=\!0$, closing the null space
by construction.
GAMP uses $\mathcal{E}(l)$ as both its per-layer sensitivity signal
and the basis of its calibration objective.

\subsection{Guidance-Aware Mixed Precision (GAMP)}
\label{sec:gamp}

Activations are substantially more sensitive than weights to bit-reduction in CFG models; GAMP exploits this asymmetry by fixing all weights to 4~bits and allocating activation bits per layer, concentrating precision where the guided output is most sensitive.
The calibration objective is:
\begin{equation}
  \mathcal{L}_{\mathrm{GAMP}} = \norm{\tileps^q - \tileps^{fp}}^2.
  \label{eq:gamp_loss}
\end{equation}
Algorithm~\ref{alg:gamp} and Fig.~\ref{fig:gamp_flow} give the full procedure.

\begin{algorithm}[t]
\caption{GAMP: Guidance-Aware Mixed Precision}
\label{alg:gamp}
\begin{algorithmic}[1]
\Require Frozen quantized network $\mathcal{M}$;
  calibration set $\mathcal{D}$ (1\,024 images);
  quantizable layers $\mathcal{L}$, $|\mathcal{L}|\!=\!49$;
  target avg.\ activation bits $\bar{b}$; guidance scale $w$
\Ensure Activation-bit assignments $\{A_l\}$; calibrated thresholds
\State Run $\mathcal{D}$ through both conditional and unconditional
  branches; record per-layer activation statistics
\State Cache $\tileps^{fp} \!\leftarrow\! \eps_\emptyset^{fp} +
  w(\eps_y^{fp}\!-\!\eps_\emptyset^{fp})$ for all $x\!\in\!\mathcal{D}$
\State Set $A_l\!\leftarrow\!8\;\forall l$; compute
  $\mathcal{E}_\text{base}
  \!\leftarrow\!
  \tfrac{1}{|\mathcal{D}|}\sum\norm{\tileps^q - \tileps^{fp}}^2$
\For{each layer $l \in \mathcal{L}$}
  \State $A_l \leftarrow 4$ temporarily
  \State $s_l \leftarrow \mathcal{E}_l^{A4} - \mathcal{E}_\text{base}$
         \hfill\textit{// guided-output sensitivity}
  \State $A_l \leftarrow 8$
\EndFor
\State Sort $\mathcal{L}$ descending by $s_l\,/\,\Delta\text{BOPs}(l)$
\State $A_l \leftarrow 4\;\forall l$
\While{$\bar{A} < \bar{b}$ \textbf{and} layers remain}
  \State Promote next-ranked layer: $A_l \leftarrow 8$
\EndWhile
\State Calibrate thresholds via $\mathcal{L}_\text{GAMP}$
  (Eq.~\eqref{eq:gamp_loss}), 800 Adam iterations
\State \Return $\{A_l\}$, calibrated thresholds
\end{algorithmic}
\end{algorithm}

Dual-branch calibration (line~1) prevents the clipping
underestimation that arises when calibrating only on unconditional
passes, since the conditional branch has systematically wider
activation ranges.
Sensitivity profiling (lines~3--7) requires $L\!=\!49$ additional
forward passes ($\approx$5~min on T4) and is a one-time offline
cost.
The greedy knapsack (lines~9--13) is near-optimal when per-layer
BOPs increments are small relative to the total budget; each
A4$\to$A8 promotion adds at most $\sim$1~G BOPs against a $\sim$32~G
total, satisfying this condition.
Calibrating with $\mathcal{L}_\text{GAMP}$ eliminates the
guided-prediction corruption that defines the branch-drift trap:
$\mathcal{L}_\text{GAMP}\!=\!0$ drives
$\norm{\tileps^q\!-\!\tileps^{fp}}\!\to\!0$, which by
Proposition~\ref{prop:error} removes the shared-drift null space
of Proposition~\ref{prop:nullspace}.

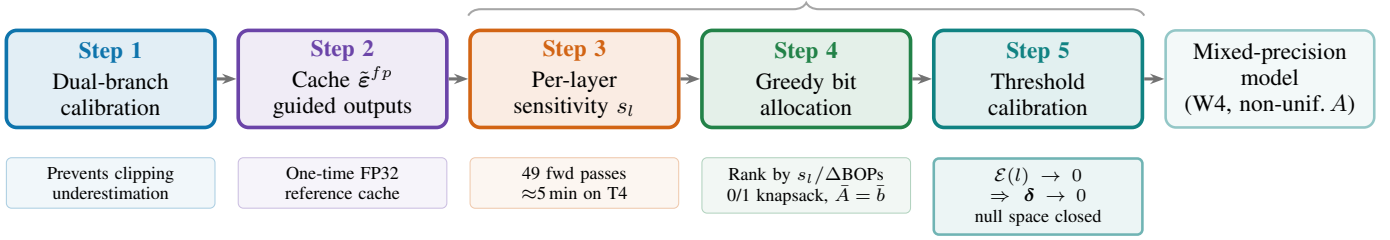
\begin{figure*}[t]
\centering
\resizebox{\linewidth}{!}{%
\begin{tikzpicture}[
  font=\small,
  s1/.style={rectangle, rounded corners=4pt,
             draw=s1col, line width=1.4pt, fill=s1col!12,
             text centered, minimum height=1.32cm,
             minimum width=2.72cm, text width=2.60cm},
  s2/.style={rectangle, rounded corners=4pt,
             draw=s2col, line width=1.4pt, fill=s2col!10,
             text centered, minimum height=1.32cm,
             minimum width=2.72cm, text width=2.60cm},
  s3/.style={rectangle, rounded corners=4pt,
             draw=s3col, line width=1.4pt, fill=s3col!12,
             text centered, minimum height=1.32cm,
             minimum width=2.72cm, text width=2.60cm},
  s4/.style={rectangle, rounded corners=4pt,
             draw=s4col, line width=1.4pt, fill=s4col!12,
             text centered, minimum height=1.32cm,
             minimum width=2.72cm, text width=2.60cm},
  s5/.style={rectangle, rounded corners=4pt,
             draw=s5col, line width=1.4pt, fill=s5col!12,
             text centered, minimum height=1.32cm,
             minimum width=2.72cm, text width=2.60cm},
  outb/.style={rectangle, rounded corners=4pt,
               draw=s5col!45, line width=1.0pt, fill=s5col!5,
               text centered, minimum height=1.32cm,
               minimum width=2.72cm, text width=2.60cm},
  a1/.style={rectangle, rounded corners=2pt,
             draw=s1col!35, fill=s1col!6, text centered,
             font=\scriptsize, minimum height=0.52cm,
             minimum width=2.72cm, text width=2.60cm},
  a2/.style={rectangle, rounded corners=2pt,
             draw=s2col!35, fill=s2col!6, text centered,
             font=\scriptsize, minimum height=0.52cm,
             minimum width=2.72cm, text width=2.60cm},
  a3/.style={rectangle, rounded corners=2pt,
             draw=s3col!35, fill=s3col!6, text centered,
             font=\scriptsize, minimum height=0.52cm,
             minimum width=2.72cm, text width=2.60cm},
  a4/.style={rectangle, rounded corners=2pt,
             draw=s4col!35, fill=s4col!6, text centered,
             font=\scriptsize, minimum height=0.52cm,
             minimum width=2.72cm, text width=2.60cm},
  a5/.style={rectangle, rounded corners=2pt,
             draw=s5col!60, line width=1pt, fill=s5col!8,
             text centered, font=\scriptsize,
             minimum height=0.65cm, minimum width=2.72cm,
             text width=2.60cm},
  arr/.style={-{Stealth[length=5.5pt,width=4pt]},
              very thick, black!55},
]

\node[s1] (n1) at ( 0.00,0)
  {{\color{s1col}\textbf{Step 1}}\\[2pt]Dual-branch\\calibration};
\node[s2] (n2) at ( 3.15,0)
  {{\color{s2col}\textbf{Step 2}}\\[2pt]Cache $\tileps^{fp}$\\guided outputs};
\node[s3] (n3) at ( 6.30,0)
  {{\color{s3col}\textbf{Step 3}}\\[2pt]Per-layer\\sensitivity $s_l$};
\node[s4] (n4) at ( 9.45,0)
  {{\color{s4col}\textbf{Step 4}}\\[2pt]Greedy bit\\allocation};
\node[s5] (n5) at (12.60,0)
  {{\color{s5col}\textbf{Step 5}}\\[2pt]Threshold\\calibration};
\node[outb](no) at (15.75,0)
  {Mixed-precision\\model\\(W4, non-unif.\,$A$)};

\foreach \a/\b in {n1/n2, n2/n3, n3/n4, n4/n5, n5/no}
  \draw[arr] (\a) -- (\b);

\node[a1, below=0.38cm of n1]
  {Prevents clipping\\underestimation};

\node[a2, below=0.38cm of n2]
  {One-time FP32\\reference cache};

\node[a3, below=0.38cm of n3]
  {49 fwd passes\\$\approx$5\,min on T4};

\node[a4, below=0.38cm of n4]
  {Rank by $s_l/\Delta\text{BOPs}$\\0/1 knapsack, $\bar{A}\!=\!\bar{b}$};

\node[a5, below=0.38cm of n5]
  {$\mathcal{E}(l)\!\to\!0$\\$\Rightarrow\dv\!\to\!0$\\
   null space closed};

\draw[decorate,
      decoration={brace,amplitude=5pt,raise=5pt},
      black!35, thick]
  (n3.north west) -- (n5.north east)
  node[midway, above=11pt, font=\scriptsize, text=black!50]
  {GAMP-specific offline overhead: $\approx$22\,min on T4};

\end{tikzpicture}}%
\caption{GAMP calibration pipeline (colors: blue~=~calibration,
  purple~=~cache, orange~=~sensitivity, green~=~allocation,
  teal~=~calibrate; output box teal-tinted).
  Dual-branch calibration (Step~1) captures the full activation
  range of both branches, preventing the clipping underestimation
  of single-pass calibration.
  Sensitivity profiling (Step~3) measures each layer's contribution
  to guided-prediction error under A4, grounded in
  $\mathcal{E}(l)$ of Eq.~\eqref{eq:guide_err}.
  Greedy allocation (Step~4) promotes layers A4$\to$A8 by
  sensitivity-per-BOPs until the target budget $\bar{b}$ is met.
  Calibrating via $\mathcal{L}_\text{GAMP}$ (Step~5, teal) removes
  the shared-drift null space, closing the branch-drift trap.
  Overhead is one-time; inference cost equals standard
  mixed-precision PTQ.}
\label{fig:gamp_flow}
\end{figure*}

\section{Experiments}
\label{sec:experiments}

\subsection{Setup}
\label{sec:setup}

The branch-drift trap established in Propositions~\ref{prop:nullspace}
and~\ref{prop:error} is a structural algebraic consequence
of~\eqref{eq:cfg}, independent of dataset size, image resolution, or
label-space cardinality: the null-space argument holds for any CFG
model that evaluates two forward passes and combines them linearly.
Likewise, the ONNX Runtime INT8 bottleneck (Section~\ref{sec:int8})
is governed by operator-mapping rules within the software stack, not
by image content.
CIFAR-10 therefore provides a sufficient and
computationally controlled setting for isolating these structural and
system-level behaviors without confounding data-scale variables.

\textit{Model and data.}
The publicly released DASH student
checkpoints~\cite{dash2026} are used: a 6.1M-parameter ADM
UNet~\cite{dhariwal2021} (base channels~64, one residual block per
scale, attention at $16\!\times\!16$) trained on CIFAR-10,
$32\!\times\!32$~RGB.
FP32 evaluation on a single-seed 50k-sample protocol yields
FID~13.67; the DASH paper reports 8.87 as a multi-seed mean from its
own training pipeline.
All PTQ comparisons use the same single-seed protocol; the
single-seed variance ($\pm$2--3 FID at 50k samples) is small
relative to the differences reported.

\textit{Baselines.}
(i)~\textsc{RTN}: round-to-nearest, no calibration;
(ii)~\textsc{MSE-cal}: AdaRound rounding with MSE activation
clipping, dual-branch calibration, uniform precision;
(iii)~\textsc{Gap-only}: gap objective~\eqref{eq:lgap} as the sole
calibration target---the critical ablation validating
Corollary~\ref{cor:metrics};
(iv)~\textsc{GAMP} (ours) at $\bar{b}\!\in\!\{5,6\}$ average
activation bits.

\textit{Metrics.}
FID and IS (50k samples, \texttt{torch-fidelity},
\texttt{cifar10-train} reference)~\cite{fid,is};
guidance-gap diagnostics $\rho$, $\cos(\Dc)$~\cite{dash2026};
BOPs per guided step ($=\!2\!\times$ single-pass BOPs);
theoretical packed model size (MB).

\textit{Hardware.}
The NVIDIA T4 (16~GB) is chosen deliberately: as the most widely
deployed cost-effective cloud and edge inference accelerator, it
ensures the 147$\times$ software-stack gap diagnosed in
Section~\ref{sec:int8} reflects conditions faced by budget-constrained
practitioners, not an unusual hardware configuration.
Timing uses PyTorch \texttt{cuda.Event} (20~reps, 8~warmup); INT8 measurements use ONNX Runtime~1.18 (CUDA EP, no TensorRT).

\subsection{CFG Latency Characterization}
\label{sec:cfgtax}

\begin{table}[t]
\caption{Measured 50-step CFG-DDIM latency on T4.
Guided\,=\,2 UNet passes/step; unguided\,=\,1.
Mem.\ is peak activation memory
(\texttt{torch.cuda.max\_memory\_allocated}, reset before each
forward pass; excludes CUDA context).}
\label{tab:profiling}
\centering
\resizebox{\columnwidth}{!}{%
\renewcommand{\arraystretch}{1.15}
\begin{tabular}{llrrrr}
\toprule
Prec. & $B$ & Guided (ms) & Unguided (ms) & Tax & Mem (MB) \\
\midrule
FP32 & 1  & 1042.6 & 534.2  & 1.95$\times$ & 594  \\
FP32 & 16 & 1106.4 & 553.6  & 2.00$\times$ & 685  \\
FP32 & 64 & 3936.2 & 1988.5 & 1.98$\times$ & 1034 \\
\midrule
FP16 & 1  & 1504.7 & 761.3  & 1.98$\times$ & 40   \\
FP16 & 16 & 1479.0 & 744.5  & 1.99$\times$ & 108  \\
FP16 & 64 & 3107.7 & 1569.8 & 1.98$\times$ & 322  \\
\bottomrule
\end{tabular}}
\end{table}

Table~\ref{tab:profiling} shows a consistent
$\approx$\textbf{1.99$\times$} CFG overhead across all batch sizes
and floating-point formats.
Fig.~\ref{fig:cfg_tax} visualizes the overhead; the shaded gap is
the pure guidance cost, invisible to all BOPs and parameter-count
metrics.
A W4A8 quantized model reporting 32.7~G guided-step BOPs carries an
implicit additional 32.7~G BOPs of guidance overhead that no prior
compression paper credits; reporting single-pass metrics overstates
the realized deployment gain by a factor of two.
FP16 incurs higher latency at batch~1 ($+$44\% vs.\ FP32) due to
kernel-launch overhead but is $1.3\times$ faster at batch~64
(31 vs.\ 39~ms), while cutting peak memory by $3.2\times$
(322 vs.\ 1034~MB).

\begin{figure}[t]
  \centering
  \includegraphics[width=\columnwidth]{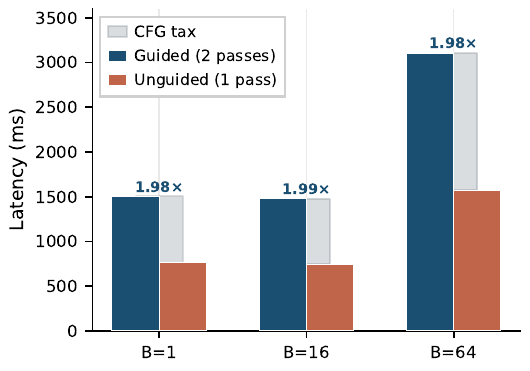}
  \caption{CFG tax: guided vs.\ unguided 50-step latency (FP16, T4)
    across batch sizes 1, 16, 64. Shaded region is the pure guidance
    overhead ($\approx$1.99$\times$), invisible to all BOPs and
    parameter-count metrics.}
  \label{fig:cfg_tax}
\end{figure}

\subsection{INT8 Deployment Gap: Software-Stack Diagnosis}
\label{sec:int8}

\begin{table}[t]
\caption{Single-pass BOPs and latency on T4 ($B\!=\!64$), without
TensorRT. Guided step $=2\times$ these values.
ORT INT8: $4598/31.2\!\approx\!\mathbf{147\times}$ slower than
PyTorch FP16.}
\label{tab:int8}
\centering
\resizebox{\columnwidth}{!}{%
\renewcommand{\arraystretch}{1.15}
\begin{tabular}{lrrrr}
\toprule
Path & Size (MB) & BOPs$_{1p}$\,(G) &
  Theo.\,$\uparrow$ & Lat.$_{1p}$\,(ms) \\
\midrule
PyTorch FP32        & 24.3 & 525  & 1.0$\times$ & 39.4  \\
PyTorch FP16        & 12.2 & 525  & 1.0$\times$ & \textbf{31.2}  \\
\midrule
ORT FP32            & 24.5 & --   & --          & 47.8  \\
ORT INT8 (no TRT)   & \textbf{6.4} & \textbf{32.8} &
  \textbf{16$\times$} & \textbf{4598} \\
\bottomrule
\end{tabular}}
{\scriptsize\raggedright
ORT INT8 is $3.8\times$ smaller on disk (6.4 vs.\ 24.5~MB):
weight quantization succeeds; the bottleneck is inference-stack
support. This 6.4~MB is the exported ONNX file size (includes
operator metadata); Table~\ref{tab:main}'s 6.1~MB is the
theoretical bit-packed weight size ($24.3/4$), a distinct
quantity.\par}
\end{table}

Table~\ref{tab:int8} documents a stark gap between BOPs-predicted and
measured INT8 efficiency.
W8A8 single-pass BOPs predict a 16$\times$ speedup over FP32; ORT
INT8 without TensorRT delivers a single-pass latency of
\textbf{4\,598~ms} against 31.2~ms for PyTorch FP16---a
$\approx$\textbf{147$\times$ slowdown} (Fig.~\ref{fig:int8}).
The root cause is diagnosed as follows.
TensorRT libraries are absent from the Kaggle T4 environment
(\texttt{libnvinfer.so.10 not found}), preventing ONNX Runtime from
dispatching quantized operators to the GPU's INT8 tensor cores.
As a result, the quantized ONNX graph contains 192
GPU$\leftrightarrow$CPU memcpy nodes (one per quantized operator),
each incurring a PCIe round-trip; additionally, UNet downsample
\texttt{Slice} operations resist dynamic INT8 quantization and
execute at FP32 within the nominally INT8 graph.

This is a software-stack finding, not a verdict on INT8 arithmetic: the $3.8\times$ disk-size reduction (6.4 vs.\ 24.5~MB) confirms weight quantization succeeds; the bottleneck is inference-stack support, not arithmetic capability.

\begin{figure}[t]
  \centering
  \includegraphics[width=\columnwidth]{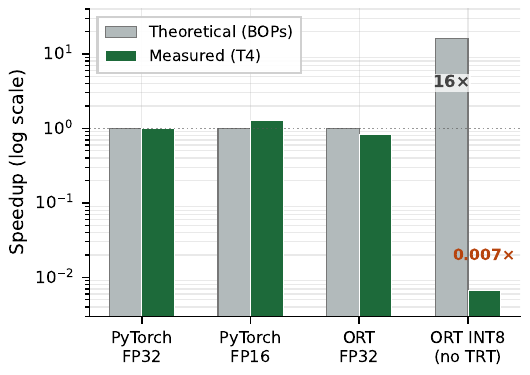}
  \caption{Theoretical speedup (BOPs ratio, grey) vs.\ measured ONNX
    Runtime speedup on T4 without TensorRT (log scale).
    W8A8 BOPs predict 16$\times$; ORT INT8 delivers
    $0.007\times$ ($\approx$147$\times$ slowdown vs.\ FP16) due to
    192 operator-fallback memcpy nodes.}
  \label{fig:int8}
\end{figure}

\subsection{Quality, Gap Fidelity, and GAMP}
\label{sec:results}

\begin{table}[t]
\caption{Quality and efficiency sweep, CIFAR-10 (50k FID, DDIM
$K\!=\!50$, timestep-adaptive guidance (TAG) $w\!=\!4.0$~\cite{dash2026},
single seed).
BOPs: per guided step ($=\!2\!\times$ single-pass).
Size: FP32 24.3~MB; W8 6.1~MB ($4\!\times$); W4 3.0~MB
($8\!\times$).
$\rho,\cos(\Dc)\!\to\!1$ and lower FID are better.
${}^\dagger$MSE-W8A8 FID/IS marginally exceed FP32 owing to
single-seed variance ($\pm$2--3 FID at 50k samples) and mild
AdaRound regularization; FP32 remains the quality ceiling.}
\label{tab:main}
\centering
\resizebox{\columnwidth}{!}{%
\renewcommand{\arraystretch}{1.22}
\begin{tabular}{lrrrrrr}
\toprule
Config & avg$A$ & $\rho$ & $\cos(\Dc)$ &
  FID$\downarrow$ & IS$\uparrow$ & BOPs\,(G) \\
\midrule
FP32                   & 32 & 1.000 & 1.000 & 13.67 & 8.61 & --   \\
\midrule
RTN-W8A8               &  8 & 1.501 & 0.612 & 48.31 & 7.18 & 65.4 \\
MSE-W8A8${}^\dagger$   &  8 & 1.252 & 0.806 & \textbf{10.53} & 8.89 & 65.4 \\
\midrule
RTN-W4A8               &  8 & 1.622 & 0.580 & 65.42 & 6.92 & 32.7 \\
MSE-W4A8               &  8 & 1.334 & 0.730 & \textbf{30.62} & 7.98 & 32.7 \\
\midrule
RTN-W4A4               &  4 & 6.277 & 0.095 & 272.78 & 1.56 & 16.4 \\
MSE-W4A4               &  4 & 3.677 & 0.169 & 122.36 & 3.74 & 16.4 \\
\midrule
Gap-only-W4A8          &  8 & \textbf{1.004} & \textbf{0.882}
                           & 334.24 & 1.26 & 32.7 \\
\midrule
GAMP-$\bar{b}$6 (ours) & 6.03 & 1.509 & 0.603
                           & \textbf{39.40} & 7.42 & 31.3 \\
GAMP-$\bar{b}$5 (ours) & 5.01 & 2.151 & 0.435
                           & 40.01 & 7.49 & 28.3 \\
\bottomrule
\end{tabular}}
\end{table}

\textbf{Finding 1 --- Activation precision cliff.}
Reducing weight precision W8$\to$W4 at fixed A8 raises FID by 20.1
(10.5$\to$30.6).
Reducing activation precision A8$\to$A4 at fixed W4 raises FID by
91.8 (30.6$\to$122.4)---a 4.6$\times$ steeper per-bit degradation.
Activations are the precision-critical resource for CFG diffusion
models, directly motivating GAMP's strategy of fixing weight bits
and allocating activation bits non-uniformly.

\textbf{Finding 2 --- Branch-drift trap confirmed.}
Gap-only-W4A8 achieves the best guidance-gap alignment of any
quantized configuration ($\rho\!=\!1.004$, $\cos(\Dc)\!=\!0.882$)
yet the worst sample quality of the entire sweep: FID~334, exceeding
even the W4A4 model (FID~272).
This is the exact false positive predicted by
Corollary~\ref{cor:metrics}: the calibration objective is satisfied
near-perfectly while the guided prediction is severely
corrupted.
The gap optimizer reaches $\mathcal{L}_\text{gap}\!\approx\!0$,
consistent with a constant branch shift $\dv$; this drift compounds across 50 DDIM
steps, producing near-random samples despite near-ideal gap
diagnostics (Fig.~\ref{fig:quality}).
The near-ideal $\rho$ and $\cos(\Dc)$ combined with severely degraded FID
is the signature Corollary~\ref{cor:metrics} predicts for
a shared-drift solution; directly instrumenting
$\norm{\eps_\emptyset^q - \eps_\emptyset^{fp}}$ alongside the gap
diagnostics would further isolate the shared-drift mechanism from
other unconstrained error patterns and is a natural extension of
this ablation.

\textbf{Finding 3 --- GAMP Pareto efficiency.}
GAMP-$\bar{b}$6 achieves FID~39.4 at 6.03 average activation bits
and 31.3~G BOPs per guided step.
A linear reference connecting the two measured uniform endpoints
(MSE-W4A4: FID~122.4; MSE-W4A8: FID~30.6) predicts
FID~$\approx$76.5 at 6 average bits; GAMP achieves 39.4, a
\textbf{49\% improvement} over this uniform-spending reference
at matched average precision.
GAMP-$\bar{b}$5 (28.3~G BOPs, 13\% below W4A8) reaches
FID~40.0---near-W4A8 quality at lower compute; the same reference
at avg$A\!=\!5$ predicts FID~$\approx$99, confirming that
guided-output sensitivity allocation substantially outperforms
uniform precision at matched average bits.
The near-identical FIDs at $\bar{b}\!\in\!\{5,6\}$ (39.4 vs.\ 40.0)
reflect rapid saturation: the most sensitivity-weighted layers are
promoted to A8 first, capturing the bulk of guided-output error
reduction within the first knapsack increments (Fig.~\ref{fig:pareto}).
Notably, GAMP achieves lower FID despite weaker gap scores
($\rho\!=\!1.509$, $\cos(\Dc)\!=\!0.603$), reinforcing that
guided-output error, not gap fidelity, governs sample quality.

\begin{figure}[t]
  \centering
  \includegraphics[width=\columnwidth]{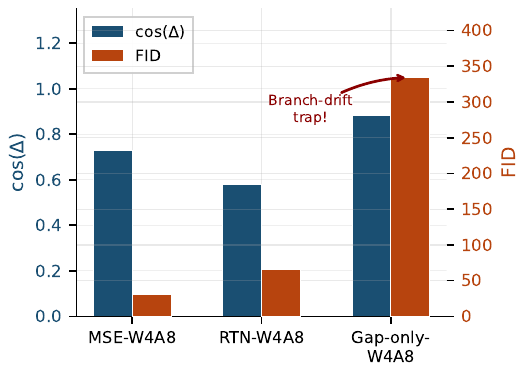}
  \caption{Guidance-gap fidelity ($\cos(\Dc)$, blue) vs.\ FID
    (orange) across W4A8 configurations.
    Gap-only achieves the highest $\cos(\Dc)$ yet the worst FID,
    empirically confirming Corollary~\ref{cor:metrics}.
    GAMP recovers near-MSE-W4A8 quality at lower BOPs.}
  \label{fig:quality}
\end{figure}

\begin{figure}[t]
  \centering
  \includegraphics[width=\columnwidth]{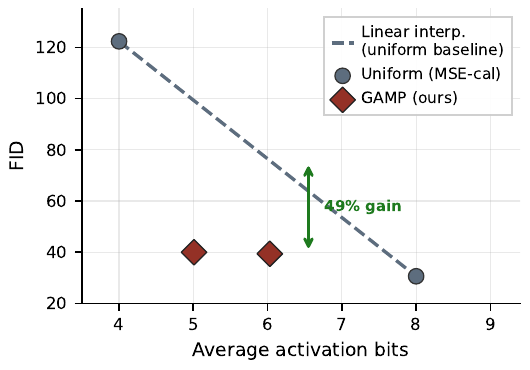}
  \caption{GAMP Pareto curve: FID vs.\ average activation bits.
    GAMP (diamonds) lie well below the linear-spending reference
    (dashed) connecting MSE-W4A4 and MSE-W4A8, demonstrating that
    guided-output sensitivity allocation substantially outperforms
    uniform spending at both tested average bit-widths.}
  \label{fig:pareto}
\end{figure}

\section{Conclusion}
\label{sec:conclusion}

This paper characterizes post-training quantization of CFG diffusion
models along three dimensions the existing literature has left
unaddressed.
Measured profiling on T4 confirms a consistent 1.99$\times$
guided-inference overhead hidden by all standard efficiency metrics.
A precise software-stack diagnosis isolates the 147$\times$ INT8
slowdown to an operator-fallback artifact rather than an arithmetic
limitation, clarifying that the BOPs-to-latency gap is a
deployment-pipeline problem.
The branch-drift trap---an exact null space in gap-only PTQ calibration, confirmed at ($\cos(\Dc)\!=\!0.882$, FID~334)---is the paper's central theoretical and empirical contribution.
GAMP closes this trap by construction and allocates activation
bits by guided-output sensitivity, achieving a 49\% FID improvement
over the uniform-spending reference at matched average precision
($\bar{b}$6) and near-W4A8 quality at 13\% fewer BOPs than the
uniform W4A8 counterpart ($\bar{b}$5).

For HPEC practitioners, three guidelines follow: report guided-step BOPs; validate PTQ calibration with guided-prediction error rather than gap diagnostics alone; and allocate activation bits by GAMP sensitivity when the activation cliff falls within the compression regime.

\textit{Limitations and future work.}
Experiments are restricted to $32\!\times\!32$ pixel-space diffusion
(CIFAR-10) on a compact 6.1M student; validating GAMP on Stable
Diffusion and similar latent CFG models, and on
Hopper-generation accelerators (e.g., H100) where INT8 tensor-core
support is more mature, is the primary scale-up direction.
The branch-drift analysis predicts the trap intensifies with $w$,
motivating ablations at guidance scales $w\!\sim\!7$--15 typical
of text-to-image generation.
Adding an unconditional anchor
$\lambda\norm{\eps_\emptyset^q - \eps_\emptyset^{fp}}^2$ to
$\mathcal{L}_\text{gap}$ closes the null space provably and is
the most direct methodological extension.
Combining GAMP with PTQD-style per-step allocation is the primary algorithmic follow-up.


\end{document}